\documentclass[10pt]{article}

\usepackage{graphicx} % Required for inserting images

\usepackage{enumitem}

\usepackage{geometry}
\AtBeginDocument{
  \newgeometry{
    textheight=9in,
    textwidth=5.5in,
    top=1in,
    headheight=12pt,
    headsep=25pt,
    footskip=30pt
  }
}

\usepackage{tikz}
\usepackage{pgfplots}
\pgfplotsset{compat=1.18}
\usepgfplotslibrary{groupplots}
\usepgfplotslibrary{fillbetween}

\usepackage{amsmath, amsthm, bm, amssymb, amsfonts}
\usepackage{physics}
\usepackage{thmtools}
\usepackage{apptools}
\AtAppendix{\counterwithin{theorem}{section}}

\usepackage{setspace}

\usepackage{microtype}
\usepackage{xcolor}

\usepackage[T1]{fontenc}
\usepackage{mathpazo}

\usepackage{hyperref}
\usepackage[round]{natbib}
\usepackage{algorithm}
\usepackage{algorithmicx}
\usepackage{algpseudocode}

\usepackage{booktabs}
\usepackage{wrapfig}
\usepackage{caption}
\usepackage{cleveref}

\crefname{assume}{assumption}{assumptions}
\Crefname{assume}{Assumption}{Assumptions}

\crefname{prop}{proposition}{propositions}
\Crefname{prop}{Proposition}{Propositions}

\crefname{condition}{condition}{conditions}
\Crefname{condition}{Condition}{Conditions}

\usetikzlibrary{decorations.markings} % <--- 1. Load library
\usetikzlibrary{arrows.meta}
\usetikzlibrary{shapes}

\tikzset{
  myarrow/.style={
    -{Latex[length=2mm, width=2.5mm]}, % controls arrowhead size
    line width=0.7pt,
    shorten >=1pt,
  }
}

\tikzset{gnode/.style={}}

\usetikzlibrary{arrows.meta}
\tikzset{
  myarrowb/.style={
    {Latex[length=2mm, width=2.5mm]}-{Latex[length=2mm, width=2.5mm]}, % controls arrowhead size
    line width=0.9pt,
    shorten >=1pt,
    shorten <=1pt
  }
}

\newtheorem{example}{Example}

\newtheorem{prop}[example]{Proposition}

\newtheorem{assume}[example]{Assumption}
\newtheorem{theorem}[example]{Theorem}
\newtheorem{condition}[example]{Condition}

\newcommand\ve{\varepsilon}

\newcommand\rvc{\mathbold{C}}
\newcommand\rve{\mathbold{\varepsilon}}
\newcommand\rvx{\mathbold{X}}
\newcommand\rvz{\mathbold{Z}}
\newcommand\rvy{\mathbold{Y}}

\newcommand\veca{\mathbold{a}}
\newcommand\vecb{\mathbold{b}}
\newcommand\vt{\mathbold{t}}

\newcommand\vx{\mathbold{x}}
\newcommand\vy{\mathbold{y}}
\newcommand\vz{\mathbold{z}}
\newcommand\vzero{\mathbold{0}}

\newcommand\vmu{\mathbold{\mu}}

\newcommand\vphi{\mathbold{\phi}}
\newcommand\vsigma{\mathbold{\sigma}}
\newcommand\vtheta{\mathbold{\theta}}

\newcommand\mA{\mathbf{A}}

\newcommand\mD{\mathbf{D}}

\newcommand\mM{\mathbf{M}}

\newcommand\mT{\mathbf{T}}
\newcommand\mW{\mathbf{W}}
\newcommand\mU{\mathbf{U}}

\newcommand\sd{\mathcal{D}}
\newcommand\se{\mathcal{E}}

\newcommand\si{\mathcal{I}}

\newcommand\sx{\mathcal{X}}
\newcommand\sy{\mathcal{Y}}

\newcommand\gr{\mathcal{G}} % Used for graphs
\newcommand\pa{\text{pa}}
\newcommand\R{\mathbb{R}}

\newcommand\id{\text{id}}

\newcommand\deff{\triangleq}
\newcommand\score{\mathcal{S}}

\newcommand\tdo{\text{do}}
\newcommand\gauss{\mathcal{N}}

\newcommand\expect{\mathbb{E}}

\title{\textbf{RECLAIM: Cyclic Causal Discovery Amid Measurement Noise}}
\author{Muralikrishnna G. Sethuraman \hspace{0.5cm} Faramarz Fekri \\ \small School of Electrical and Computer Engineering, Georgia Institute of Technology}
\date{}

\begin{document}

\maketitle

% \blfootnote{Preprint Under Review}

\begin{abstract}
Uncovering causal relationships is a fundamental problem across science and engineering. However, most existing causal discovery methods assume acyclicity and direct access to the system variables---assumptions that fail to hold in many real-world settings. For instance, in genomics, cyclic regulatory networks are common, and measurements are often corrupted by instrumental noise. To address these challenges, we propose RECLAIM, a causal discovery framework that natively handles both cycles and measurement noise. RECLAIM learns the causal graph structure by maximizing the likelihood of the observed measurements via expectation-maximization (EM), using residual normalizing flows for tractable likelihood computation. We consider two measurement models: (i) Gaussian additive noise, and (ii) a linear measurement system with additive Gaussian noise. We provide theoretical consistency guarantees for both the settings. Experiments on synthetic data and real-world protein signaling datasets demonstrate the efficacy of the proposed method.
\end{abstract}

\section{Introduction}
\label{sec:introduction}

% Why causal discovery and what is it?
Understanding cause-effect relationships among variables is a central problem in science and 
engineering~\citep{sachs_causal_2005,friedman2000using}. Causal models provide a mechanistic 
view of a system, enabling prediction of its behavior under unseen interventions. Causal 
relations are typically encoded as a directed graph (DG), where edges represent causal 
dependencies, and learning them is naturally posed as a graph structure learning problem.

% Brief overview of the existing causal discovery methods
Causal discovery methods broadly fall into three categories: (i) \emph{constraint-based}, 
(ii) \emph{score-based}, and (iii) \emph{hybrid}. Constraint-based methods 
\citep{spirtes2000causation,triantafillou2015constraint} identify causal graphs consistent 
with the conditional independence constraints observed in data, but scale poorly as the 
number of required independence tests grows exponentially. Score-based methods 
\citep{hauser2012characterization} instead optimize a scoring function 
such as the Bayesian Information Criterion (BIC) over the space of graphs; greedy search 
is typically employed since the search space grows super-exponentially with the number of 
nodes. Hybrid methods~\citep{tsamardinos2006max,solus2021consistency} combine elements of 
both approaches.

% Drawback with existing methods
With few exceptions, most causal discovery methods rely on two key assumptions: (i) the 
underlying causal graph is acyclic (a directed acyclic graph, or DAG), and (ii) the system 
variables are directly observed. While these assumptions simplify the learning problem, they 
are often violated in practice. Variables of interest are frequently \emph{latent}---for 
instance, a person's beliefs cannot be measured directly, but surveys provide noisy proxies 
of the underlying state. Furthermore, feedback loops are prevalent in biological 
systems~\citep{freimer2022systematic}, and such systems often exhibit structured measurement 
noise such as ``dropout.'' Imposing these assumptions in such settings can lead to 
misleading causal conclusions, limiting the practical applicability of existing methods.

% Our proposed method
To address these challenges, we propose \textsc{Reclaim}, a novel causal discovery framework 
that learns latent causal structure while jointly modeling directed cycles and measurement 
noise. We consider two measurement noise models: (i) Gaussian additive noise, and (ii) a 
linear measurement system with Gaussian noise. Using interventional data, we first estimate 
the noise distribution parameters, then employ an expectation-maximization (EM) procedure 
to learn the causal graph via likelihood maximization.

\subsection{Related Works}

% Cyclic causal discovery
\paragraph{Cyclic causal discovery.} The prior work on cyclic causal discovery spans constraint-based \citep{richardson1996discovery}, ICA-based \citep{lacerda2012discovering}, and score-based \citep{huetter2020estimation,amendola2020structure} approaches. Several methods also leverage interventional data for structure recovery \citep{hyttinen2012learning,huetter2020estimation}. Most relevant to our work, \citet{sethuraman2023nodags} proposed a differentiable framework for nonlinear cyclic graphs that models data likelihood directly, bypassing acyclicity constraints. \citet{sethuraman2025differentiable} later extended this to account for latent confounders. 

% Causal discovery under measurement noise
\paragraph{Causal discovery from indirect measurements.}
Under indirect measurements, causal discovery broadly splits into two paradigms.
\emph{Causal representation learning} (CRL) aims to jointly recover latent variables and their causal structure from high-dimensional observations under unknown transformations. Identifiability has been studied under parametric assumptions such as linear models \citep{squires2023linear,buchholz2023learning} and polynomial transformations \citep{ahuja2024multi}, as well as in fully nonparametric settings \citep{von2023nonparametric}. \emph{Causal discovery under measurement noise}, by contrast, assumes a known structural relationship between latents and observations, and focuses solely on recovering the causal graph over the latent variables. \citet{harris2013pc,yoon2020sparse} established consistency of the PC algorithm under Gaussian measurement error, and \citet{saeed2020anchored} extended this to a general constraint-based framework accommodating additive noise and dropouts. In the cyclic regime, \citet{sethuraman2024missnodag} address dropout noise by treating zeroed-out variables as missing data.

% Summary of our contributions
\subsection{Contributions}

In this work, we address two central limitations in causal discovery: (i) inability to handle measurement error, and (ii) restriction to acyclic graphs. Our main contributions are:

\begin{itemize}
    \item We introduce \textsc{Reclaim}, a differentiable causal discovery framework that natively handles nonlinear cyclic relationships in the presence of additive Gaussian measurement error.
    \item We consider two measurement noise models: (i) Gaussian additive noise and (ii) linear measurement systems, and provide a consistent estimator for the noise variance when appropriate interventions are available.
    \item We establish that exact maximization of the proposed score function under appropriate interventions identifies the interventional Markov equivalence class of the ground-truth graph.
    \item We conduct extensive experiments comparing \textsc{Reclaim} against the state-of-the-art causal discovery methods on both synthetic and real-world datasets, and show that increasing the number of measurements results in improved recovery performance.
\end{itemize}

\subsubsection{Organization.}
The remainder of the paper is organized as follows.
\Cref{sec:problem-setup} introduces the problem setup and the measurement models under consideration.
\Cref{sec:reclaim} presents the technical details of \textsc{Reclaim}.
\Cref{sec:experiments} evaluates its effectiveness on synthetic and real-world benchmarks.
\Cref{sec:discussions} concludes the paper.

\section{Problem Setup}
\label{sec:problem-setup}

\subsection{Structural Causal Model}
\label{ssec:sem}

% Preliminaries - graph, SCM, exogenous variables
Let $\gr = (\sx, \se)$ denote a directed graph with vertex set $\sx = \{X_1, \ldots, X_d\}$ and edge set $\se \subseteq \sx \times \sx$. Let $\rvx = (X_1, \ldots, X_d)$ be the associated random (endogenous) variable. The edge $(X_i, X_j) \in \se$ denotes a directed edge originated from the node $X_i$ to the node $X_j$, with slight abuse of notation we also denote this edge as $X_i \to X_j$. Following \cite{bollen1989structural,pearl2009causality}, we utilize the framework of structural causal models (SCM) to represent the functional dependencies in the causal graph. For each $i \in \sx$, we have the following \emph{structural equation}:
\begin{equation}
    \label{eq:sem-individual}
    X_i = f_i(\pa_\gr(X_i)) + Z_i,
\end{equation}
where $f_i$, called the \emph{causal mechanism}, encodes the functional dependency between the parents and the children. $\pa_\gr(X_i) \deff \{X_j\in \sx \mid X_j \to X_i \in \se\}$ denotes the \emph{parent set} of a node $X_i \in \sx$. 
% Exogenous noise and causal sufficiency
The set of variables $\rvz = (Z_1, \ldots, Z_d)$ represent the \emph{exogenous noise} within the system. We assume that the exogenous noise variables are independent of each other, i.e., $Z_i \perp Z_j$ for $i, j \in [d]$ and $i \neq j$. This assumption, also known as \emph{causal sufficiency}, excludes hidden confounding and selection bias. Let $f = (f_1, \ldots, f_d)$ be the combined causal mechanism, combining \cref{eq:sem-individual} for all $i \in [d]$, we obtain the following: 
\begin{equation}
\label{eq:sem-combined}
    \rvx = f(\rvx) + \rvz.
\end{equation}

% Solution is not a gurantee and forward map diffeomorphism
However, it is not guaranteed that \cref{eq:sem-combined} has a solution, primarily due to the (potential) presence of cycles in $\gr$. As a result, restrictions are necessary on the causal mechanism to ensure the system attains equilibrium (see \cref{sec:solution-to-sem} for a detailed discussion). To that end, building on~\cite{sethuraman2023nodags}, we make the following assumption on the map $(\id - f): \rvx \mapsto \rvz$, which we call the \emph{forward map}.

\begin{assume}\label{assume:forward-map-diffeo}
    The forward map $(\id - f): \rvx \mapsto \rvz$ is a \emph{diffeomorphism}\footnote{A function $f$ is a diffeomorphism if $f^{-1}$ exists, and both $f$ and $f^{-1}$ are differentiable.}.
\end{assume}
  
% Probability of the observations
Finally, the SCM provides us with a probability density for the exogenous noise variables, which we denote as $p_Z$. Given $p_Z$, we can obtain the probability density of the endogenous variables using the change of variable property of probability densities. That is,
\begin{equation}
    p_\gr(\rvx) = p_Z\big((\id - f)(\rvx)\big)\big|J_{(\id - f)}(\rvz)\big|,
\end{equation}
where $J_{(\id - f)}(\rvx)$ is the Jacobian matrix of the forward map evaluated at $\rvx$.

\paragraph{Interventions. }

We consider \emph{surgical interventions}~\citep{pearl2009causality} within our framework. Under surgical interventions, when a subset of nodes $\sx_I \subseteq \sx$ are (surgically) intervened, we obtain a mutilated graph, denoted as $\tdo(I)(\gr)$, where all the incoming edges to the intervened nodes are removed  (see \cref{fig:measurement-process-graph}). We identify an intervention by the index set $I \subseteq [d]$ of the nodes intervened in $\sx$. Given a set of intervened nodes $\sx_I \subseteq \sx$, let $\mU \in \{0,1\}^{d\times d}$ be a diagonal masking matrix with $U_{kk} = 1$ if $X_k \in \sx_I$, and $0$ otherwise. The structural equations in \cref{eq:sem-combined} are then modified as follows: 
\begin{equation}\label{eq:sem-intervention}
    \rvx = \mU f(\rvx) + \mU \rvz + \rvc,
\end{equation}
where $\rvc \in \R^d$ denotes the values of the intervened nodes. Consequently, under the intervention $I$, the probability density of the observations is given by
\begin{equation}\label{eq:prob-density-interventions}
    p_{\tdo(I)(\gr)}(\rvx) = p_C(\rvx_I)p_Z\big(\rvz_{[d]\setminus I}\big)\big|J_{(\id - \mU f)}(\rvz)\big|,
\end{equation}
where $p_C$ denotes the interventional density, and $\rvz_{[d] \setminus I} = \big[(\id - \mU f)(\rvx)\big]_{[d]\setminus I}$. In our experiments, we assume that the intervened nodes are sampled from a distribution with known variance $\sigma_\si^2$.

\subsection{Measurement System}
\label{sec:measurement-sys}

In practice, the causal variables $\sx$ are latent and only indirectly observed. They often pass through a measurement system, yielding \emph{measured} variables $\sy = \{Y_1, \ldots, Y_p\}$ (with $\rvy$ being the corresponding random variable). Each $Y_j$ represents a coarsened measurement of $\rvx$ given by, 
\begin{equation}\label{eq:measurements}
    Y_j = g_j(\rvx, \ve_j), \quad j = 1, \ldots, p,
\end{equation}
where $g = (g_1, \ldots, g_p)$ denotes the \emph{coarsening mechanism} (also referred as the \emph{measurement channel}), and $\rve = (\varepsilon_1, \ldots, \varepsilon_d)$ represents any external factor that can affect the measurement system. We assume that $p \geq d$ (otherwise, identifiability is ill-posed), and $\varepsilon_i \perp \varepsilon_j$, for $i, j \in [p]$ and $i\neq j$. This results in a combined measurement graph, denoted as $\gr_m$, whose vertex set contains both the latent variables $\sx$ and the measured variables $\sy$, i.e., $\sx \cup \sy$, see \cref{fig:measurement-process-graph}. 

In our work, we deal with two different types of measurement system: 
\begin{enumerate}[itemsep=1em]
    \item \textbf{Gaussian additive noise}: In this case, $p = d$, $\ve_i \sim \gauss(0, \sigma_i^2)$, and 
    \begin{equation}\label{eq:gan-measurements}
    Y_i = X_i + \ve_i, \quad i = 1, \ldots, d.
    \end{equation}
    Given the latent variables $\rvx$, the density of the measurements is also Gaussian, i.e., $\rvy \mid \rvx = \vx \sim \gauss(\vx, \mD_\ve)$, where $\mD_\ve = \text{Diag}(\sigma_1^2, \ldots, \sigma_d^2)$. We refer to $p_{Y|X}(\rvy\mid \rvx)$ as the \emph{measurement channel density}. 
    
    \item \textbf{Linear measurement system with Gaussian noise}: Similar to the previous setting, we have $\ve_j \sim \gauss(\vzero, \sigma_j^2)$. However, in this case, $p$ can be larger than $d$, i.e., $p \geq d$, and 
    \begin{equation}\label{eq:linear-measurements}
    Y_j = \sum_{i=1}^d A_{ji}X_i + \ve_j.     
    \end{equation}
    The matrix $\mA = [A_{ji}]$, with $i \in [d]$ and $j \in [p]$, is referred to as the \emph{measurement matrix}, and is assumed to be known and full (column) rank. Given the latent variables $\rvx$, the conditional density of $\rvy$ is again Gaussian, i.e., $\rvy \mid \rvx = \vx \sim \gauss(\mA\vx, \mD_\ve)$, where $\mD_\ve = \text{Diag}(\sigma_1^2, \ldots, \sigma_p^2)$.
\end{enumerate}

\noindent
\textbf{Goal.}  Given a set of interventions $\si = \{I_k\}_{k=1}^K$, \emph{our goal is to learn structure of $\gr$ by maximizing the log-likelihood of the coarsened measurements}. We present our approach towards addressing this problem in the next section.  

\begin{figure}[t]
    \centering
    \scalebox{0.9}{
\begin{tikzpicture}
% \sansmath 
% \sffamily
\node (x1) [circle, draw=black!60, fill=black!10] at (0,0) {$X_1$};
\node (x2) [circle, draw=black!60, fill=black!10] at (1.5,-0.7) {$X_2$};
\node (x3) [circle, draw=black!60, fill=black!10] at (3,0) {$X_3$};

\node (y1) [circle, draw=black] at (-1,-2.1) {$Y_1$};
\node (y2) [circle, draw=black] at (0,-2.1) {$Y_2$};
\node (y3) [circle, draw=black] at (1.5,-2.1) {$Y_3$};
\node (y4) [circle, draw=black] at (3,-2.1) {$Y_4$};
\node (y5) [circle, draw=black] at (4,-2.1) {$Y_5$};

\node (a) at (1.5, -3) {(a)}; 

\draw [myarrow] (x1) -- (x2);  
\draw [myarrow] (x2) -- (x3);  
\draw [myarrow] (x3) -- (x1);

\draw [myarrow, blue] (x1) -- (y1);
\draw [myarrow, blue] (x1) -- (y2);
\draw [myarrow, blue] (x2) -- (y2);
\draw [myarrow, blue] (x2) -- (y3);
\draw [myarrow, blue] (x2) -- (y4);
\draw [myarrow, blue] (x3) -- (y3);
\draw [myarrow, blue] (x3) -- (y5); 

\begin{scope}[xshift=8cm]
\node (xi1) [circle, draw=black!60, fill=black!10] at (0,0) {$X_1$};
\node (xi2) [circle, draw=red!60, fill=red!10] at (1.5,-0.7) {$X_2$};
\node (xi3) [circle, draw=black!60, fill=black!10] at (3,0) {$X_3$};

\node (yi1) [circle, draw=black] at (-1,-2.1) {$Y_1$};
\node (yi2) [circle, draw=black] at (0,-2.1) {$Y_2$};
\node (yi3) [circle, draw=black] at (1.5,-2.1) {$Y_3$};
\node (yi4) [circle, draw=black] at (3,-2.1) {$Y_4$};
\node (yi5) [circle, draw=black] at (4,-2.1) {$Y_5$};

\node (b) at (1.5, -3) {(b)}; 
\end{scope}

\draw [myarrow] (xi2) -- (xi3);  
\draw [myarrow] (xi3) -- (xi1);

\draw [myarrow, blue] (xi1) -- (yi1);
\draw [myarrow, blue] (xi1) -- (yi2);
\draw [myarrow, blue] (xi2) -- (yi2);
\draw [myarrow, blue] (xi2) -- (yi3);
\draw [myarrow, blue] (xi2) -- (yi4);
\draw [myarrow, blue] (xi3) -- (yi3);
\draw [myarrow, blue] (xi3) -- (yi5); 
\end{tikzpicture}}
    \vspace{-0.3cm}
    \caption{(a) Illustration of a causal graph $\gr_m$ encoding both the system variables $\sx$ and the measured variables $\sy$ for a linear measurement system. (b) Mutilated graph, $\tdo(I)(\gr_m)$, resulting from a surgical intervention on $X_2$. }
    \label{fig:measurement-process-graph}
\end{figure}
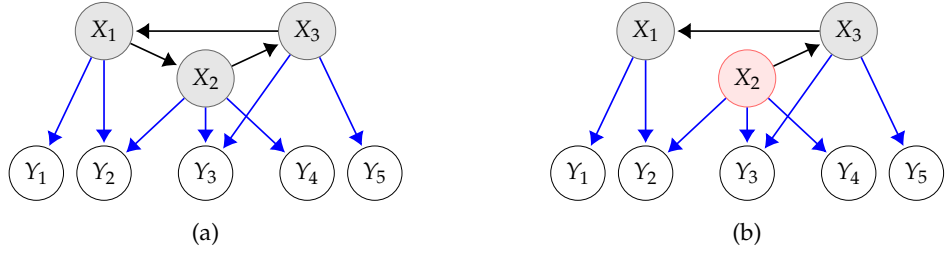

\section{RECLAIM Causal Discovery Framework}
\label{sec:reclaim}

Let $\vtheta, \vphi$ represent the parameters of the latent variable density $p_\gr(\rvx)$ and the conditional measurement density $p_{Y|X}(\rvy\mid\rvx)$, respectively. For an interventional experiment $I \subseteq [d]$, we aim to learn the causal graph structure $\gr$ by maximizing the likelihood of the observed variables. However, this maximization is generally intractable, since
$$
p(\rvy\mid \vtheta, \vphi, I) = \int p_{\tdo(I)(\gr)}(\rvx\mid \vtheta)\, p_{Y|X}(\rvy\mid \rvx, \vphi)\, d\rvx
$$
does not admit a closed-form solution in general.

To address this intractability, RECLAIM adopts an EM-based approach that maximizes a tractable lower bound on $\log p(\rvy \mid \vtheta, \vphi)$. 
We now describe the framework in detail, beginning with defining the score function to be maximized, followed by an overview of the EM-based optimization procedure. We then detail the technical components of the expectation step: estimating the measurement process parameters, computing the latent likelihood, and sampling latents conditioned on the observations. Finally, we establish the theoretical guarantees of RECLAIM, showing that exact maximization of the score function recovers the true causal structure up to Markov equivalence.

\subsection{Score Function}
\label{sec:reclaim-score}

Given a family of interventions $\si = \{I_k\}_{k=1}^K$ and a causal graph $\gr$, we define the \emph{score function} as the regularized log-likelihood of the observed variables under each interventional regime. Following~\cite{sethuraman2025differentiable}, this takes the form:
\begin{equation}\label{eq:score-large-sample}
    \score_\si(\gr) \triangleq \sup_{\vtheta, \vphi} \sum_{k=1}^K \expect_{\rvy \sim p^{(k)}} 
    \log p(\rvy \mid \vtheta, \vphi, I_k) - \lambda|\gr|,
\end{equation}
where $p^{(k)}$ denotes the data-generating density under the $k$-th interventional experiment $I_k$, and $\lambda|\gr|$ is an $\ell_0$-type sparsity penalty on the edge set of $\gr$.

In practice, we assume access to a collection of finite samples per interventional experiments, $\sd_\si = \{\{\vy^{(k, \ell)}\}_{\ell=1}^{n_k}\}_{k=1}^K$. Thus, the expectation in \cref{eq:score-large-sample} is replaced with sample mean. Finally, we maximize the score function over the space of graphs $\gr$ resulting in the following score function: 
\begin{equation}\label{eq:score-finite-sample}
    \tilde\score_\si(\sd_\si) \triangleq \sup_{\gr, \vtheta, \vphi} \sum_{k=1}^K \sum_{\ell=1}^{n_k} 
    \log p(\vy^{(k, \ell)} \mid \vtheta, \vphi, I_k) - \lambda|\gr|.
\end{equation}

\subsection{Optimizing The Score Via Penalized Expectation-Maximization}

As discussed earlier, the intractability of \cref{eq:score-finite-sample} stems from integrating out the latent variables, precluding a closed-form solution. To circumvent this, we adopt the penalized EM framework of \cite{chen2014penalized}. Beginning from an initial guess $\vtheta^0$ for the latent SCM parameters (encompassing both the neural network weights and the graph adjacency), the algorithm alternates between the two steps below until a convergence criterion is met:

\begin{description}[itemsep=1.5em, leftmargin=0em]
    \item[E-step:] Given the current parameter estimate $\vtheta^t$ and the observations, form the surrogate objective $\mathcal{Q}(\vtheta\mid\vtheta^t)$ as the expected complete-data log-likelihood:
    \begin{equation}\label{eq:e-step}
        \mathcal{Q}(\vtheta \mid \vtheta^t) = \sum_{k=1}^K\sum_{\ell=1}^{n_k} \expect_{\rvx \sim p(\cdot \mid \vy^{(k, \ell)}, \vtheta^t, \hat\vphi, I_k)} \Big[\log p\big(\rvx, \vy^{(k, \ell)}\mid \vtheta, \hat\vphi, I_k\big) \Big].
    \end{equation}
    where $\hat\vphi$ denotes the estimated parameters of the measurement channel (see \cref{sec:measurement-system-estimate} for details), and
    \begin{equation}\label{eq:factorization}
        \log p\big(\rvx, \vy^{(k, \ell)}\mid \vtheta, \hat\vphi, I_k\big) = \log p_{\tdo(I_k)(\gr)}(\rvx\mid \vtheta) + \log p\big(\vy^{(k, \ell)}\mid \rvx, \hat\vphi\big).
    \end{equation}

    \item[M-step:] Update the parameters by solving the penalized maximization of the surrogate:
    \begin{equation}\label{eq:m-step}
        \vtheta^{t+1} = \arg\max_{\vtheta}\; \mathcal{Q}(\vtheta \mid \vtheta^t) - \lambda\,\mathcal{R}(\vtheta),
    \end{equation}
    where $\mathcal{R}(\vtheta)$ is a sparsity-promoting regularizer on the graph.
\end{description}

The maximization in the M-step is carried out with stochastic gradient ascent. A key observation is that the surrogate $\mathcal{Q}$ serves as a lower bound (up to a constant) on the marginal log-likelihood:
\begin{equation}\label{eq:elbo}
    \sum_{k=1}^K \sum_{\ell=1}^{n_k}
    \log p\!\left(\vy^{(k, \ell)} \mid \vtheta, \hat\vphi, I_k\right)
    \;\geq\; \mathcal{Q}(\vtheta \mid \vtheta^t) - \text{const}.
\end{equation}
Consequently, each M-step improves a valid lower bound on the log-likelihood of the observed-data rather than optimizing it directly. A formal derivation of \cref{eq:elbo} and a convergence analysis are deferred to \cref{app:convergence}.

\subsection{Estimation of Measurement System Parameters}
\label{sec:measurement-system-estimate}

Computing the surrogate objective $\mathcal{Q}$ in \cref{eq:e-step} requires evaluating the 
measurement channel density, whose parameters are a priori unknown and must therefore be 
estimated from data. In the two measurment processes under consideration (Gaussian additive noise and linear measurement system), the unknown parameters correspond to the variances of the additive noise, 
which fully identify the measurement channel. 
% For dropout noise, the dropout probabilities $\vpi = (\pi_1, \ldots, \pi_p)$ constitute an additional set of unknowns. 
Since the latent 
variables $\rvx$ are not directly observed, estimating these parameters requires additional 
structure on the data-generating mechanism. We formalize this as the following condition on 
the intervention family $\si$.

\begin{condition}[Measurement Channel Identifiability]\label{cond:channel-indentifiability-interventions}
An intervention family $\si = \{I_k\}_{k=1}^K$ is \emph{measurement channel identifiable} 
if, for each node $X_i \in \sx$, there exists an intervention $I_k \in \si$ such that 
$X_i \in \sx_{I_k}$.
\end{condition}

Intuitively, \cref{cond:channel-indentifiability-interventions} requires that every latent 
variable is targeted by at least one intervention, ensuring that its variance is known under 
that interventional regime and can therefore be used to identify the measurement noise. We 
now show that this condition is sufficient to identify the measurement parameters for each 
of the three noise models under consideration.

\subsubsection{Gaussian Additive Noise.}

Since $\rvx$ and $\rve$ are independent, the marginal variance of $Y_i$ decomposes as 
$\sigma_{Y_i}^2 = \sigma_{X_i}^2 + \sigma_i^2$, where $\sigma_{Y_i}^2$ can be estimated 
directly from observations. While $\sigma_{X_i}^2$ is generally unknown under observational 
data, \cref{cond:channel-indentifiability-interventions} guarantees the existence of an 
intervention $I_k \in \si$ under which $X_i$ is intervened upon, fixing its variance to the 
known interventional variance $\sigma_\si^2$. The measurement noise variance is then 
recovered as $\sigma_i^2 = \sigma_{Y_i}^2 - \sigma_\si^2$.

\subsubsection{Linear Measurement System with Gaussian Noise.}

In this setting, each observed variable $Y_j$ is a linear mixture of all latent variables 
(cf.\ \cref{eq:linear-measurements}), making direct variance decomposition infeasible. 
Notably, \cref{cond:channel-indentifiability-interventions} does not require a single 
intervention targeting all nodes simultaneously. Instead, we exploit the structure of the 
measurement matrix $\mA$ via the following proposition.

\begin{prop}\label{prop:projection-existence}
    Let $\mA \in \R^{p\times d}$ be the measurement matrix with $\rank(\mA) = d$, let 
    $\veca_i$ denote the $i$-th column of $\mA$, and let $\mA_{-i}$ denote the measurement 
    matrix excluding the $i$-th column. Then, there exists a vector $\vt \in \R^p$ such 
    that $\mA_{-i}^\top \vt = \bm{0}$ and $\veca_i^\top \vt \neq 0$. 
\end{prop}

We defer the proof to \cref{app:prop-projection-existence-proof}. The key idea is to project the 
observations onto a direction $\vt_i$ that isolates the contribution of $X_i$. Specifically, 
choosing $\vt_i$ such that $\mA_{-i}^\top \vt_i = \bm{0}$ and $\veca_i^\top \vt_i \neq 0$, 
and letting $\zeta_i = \vt_i^\top \rvy$, we obtain:
\begin{align*}
    \zeta_i = \vt_i^\top \rvy = \vt_i^\top \mA\rvx + \vt_i^\top \rve 
    = (\vt_i^\top \veca_i) X_i + \vt_i^\top \rve.
\end{align*}
Since $\rvx$ and $\rve$ are independent, the variance of $\zeta_i$ decomposes as:
\begin{equation}
    \label{eq:projection-sigma}
    \sigma_{\zeta_i}^2 = (\vt_i^\top \veca_i)^2 \sigma_{X_i}^2 + \vt_i^\top \mD_\ve \vt_i.
\end{equation}
\Cref{cond:channel-indentifiability-interventions} guarantees the existence of an intervnetion $I_k$ containing $X_i$, and $\sigma_{X_i}^2 = \sigma_\si^2$. Letting $\vt_i^\odot \triangleq \vt_i \odot \vt_i$, $b_i = \sigma_{\zeta_i}^2 - 
(\vt_i^\top \veca_i)^2 \sigma_{\si}^2$, and $\vsigma^2 = (\sigma_1^2, \ldots, \sigma_p^2)$, 
\cref{eq:projection-sigma} reduces to the linear equation $(\vt_i^\odot)^\top \vsigma^2 = b_i$.
Aggregating sufficiently many such projection vectors into a matrix $\mT_2 = [(\vt^{(k)})_k^\top]$, 
the noise variances are recovered by solving the convex least-squares problem:
\begin{equation}
    \label{eq:noise-parameter-estimation}
    \min_{\vsigma^2 \geq 0} \norm{\mT_2 \vsigma^2 - \vecb}_2^2.
\end{equation}
This approach requires $\mT_2$ to have rank $p$, which may not always be achievable. However, 
the following theorem shows that this is a generic property of the measurement matrix.

\begin{theorem}\label{theorem:full-rank-possible}
Let $\mA \in \R^{p\times d}$ with $p \geq d$ and $\rank(\mA) = d$. For each $i \in [d]$, 
let $S_i(\mA) \triangleq \mathcal{N}(\mA_{-i}^\top)$, and define the admissible projection set as
$$\mathcal{T}(\mA) \triangleq \bigcup_{i=1}^d \{\vt \in S_i(\mA): \veca_i^\top \vt \neq 0\}.$$
Let $\bar{\mathcal{A}}$ denote the set of full-rank matrices $\mA$ for which no choice of 
$\vt^{(1)}, \ldots, \vt^{(p)} \in \mathcal{T}(\mA)$ yields $\rank(\mT_2) = p$. Then 
$\bar{\mathcal{A}}$ has zero Lebesgue measure.
\end{theorem}

We defer the proof to \cref{app:theorem-full-rank-possible-proof}. The projection 
vectors $\vt$'s are sampled to construct $\mT_2$ using the procedure detailed in 
\cref{app-alg:projection-sampling} in \cref{app:projection-sampling}. 

\subsection{Computing Likelihood of Latent Variables}
\label{sec:latent-likelihood}

From \cref{eq:e-step,eq:factorization}, computing the surrogate loss function $\mathcal{Q}$ requires evaluating log likelihood of the latent variables, $\log p_{\tdo(I_k)(\gr)}(\rvx \mid \vtheta)$. We describe how this is computed by detailing the SCM parameterization and the associated log-likelihood.

\paragraph{Modeling the Causal Mechanism.} 

We model each causal mechanism $f_i$ in \cref{eq:sem-individual} using a contractive neural network, that is, a neural network with Lipschitz constant less than one, which is enforced during training by rescaling layer weights by their spectral norm~\citep{iresnet}. Contractivity guarantees, via the Banach fixed point theorem~\citep{banachfp}, that the map $(\id - \mU f)$ is invertible (ensuring that \cref{assume:forward-map-diffeo} is satisfied) and the log-determinant of its Jacobian is well-defined.

To prevent spurious self-loops and encourage sparsity, we introduce a dependency mask 
$\mM \in \{0,1\}^{d \times d}$ with zero diagonal, whose entries are treated as Bernoulli 
random variables sampled via the Gumbel-softmax distribution~\citep{jang2016categorical}, $\mM \sim p_M(\mM\mid \vtheta)$. 
The causal mechanism then takes the form:
\begin{equation}
    [\mathrm{F}_{\vtheta}(\rvx)]_i = [\mathrm{NN}_{\vtheta}(\mM_{\ast,i} \odot \rvx)]_i,
    \label{eq:sem-func-model}
\end{equation}
where $\mM_{\ast,i}$ denotes the $i$-th column of $\mM$. and the sparsity penalty is given by 
$\mathcal{R}(\vtheta) = \mathbb{E}_{\mM\sim p_M(\cdot\mid \vtheta)}\|\mM\|_1$.

\paragraph{Log-Determinant of the Jacobian.} 

Given the contractive neural network parameterization in \cref{eq:sem-func-model}, the log-likelihood in \cref{eq:e-step} requires evaluating the log-determinant of the Jacobian of $(\id - \mU f)$. Na\"ive computation of the log-determinant of the Jacobian function has a computation cost that is $\mathcal{O}(d^3)$. However, since the causal mechanism $f$ is contractive, the log-determinant admits the convergent power series expansion~\citep{iresnet}:
\begin{equation}
    \log \big|\det J_{(\id - \mU f)}(\rvx)\big| = 
    -\sum_{m=1}^\infty \frac{1}{m} \mathrm{Tr}\Big\{J^m_{\mU f}(\rvx)\Big\}.
    \label{eq:power-series}
\end{equation}
The terms in the power series only depend on the trace of the powers of the Jacobian matrix, and hence reduces the computation cost to $\mathcal{O}(d^2)$. However, Truncating this series introduces bias, which we correct using the Russian roulette estimator~\citep{russianroulette}, and reduce the per-iteration cost using the Hutchinson trace estimator~\citep{hutchtraceestimator}. This results in the final unbiased estimator given by \cref{app-eq:unbiased-log-det} in \cref{app:log-det-jacobian}. We defer the details to \cref{app:log-det-jacobian}.

\subsection{Sampling Latents Given the Observations}

With the measurement parameters $\hat\vphi$ estimated and the latent likelihood in hand, the remaining challenge in evaluating $\mathcal{Q}$ is computing the posterior expectation over $\rvx$. Since the nonlinear causal mechanisms preclude a closed-form posterior, we employ \emph{Sampling Importance Resampling} (SIR)~\citep{smith1992bayesian} to draw approximate posterior samples.

We draw $S$ proposal samples $\{\vx^{(s)}\}_{s=1}^S$ from a Gaussian proposal 
$q(\rvx \mid \vy) = \mathcal{N}(\vmu_y, \mD_\ve)$, where the proposal mean $\vmu_y$ 
depends on the measurement model:
\begin{equation}
    \vmu_y = \begin{cases} \vy & \text{additive Gaussian and dropout,} \\ 
    (\mA^\top \mA)^{-1} \mA^\top \vy & \text{linear measurement system.} \end{cases}
\end{equation}
Each sample is then assigned an importance weight proportional to:
\begin{equation}
    w^{(s)} \propto \frac{p(\vx^{(s)} \mid \vtheta)\, p(\vy \mid \vx^{(s)}, \vtheta)}
    {q(\vx^{(s)} \mid \vy)},
    \label{eq:sir-weights}
\end{equation}
and a final set of samples is obtained by resampling from $\{\rvx^{(s)}\}_{s=1}^S$ 
according to the normalized weights $\{\tilde{w}^{(s)}\}_{s=1}^S$.

\subsection{Consistency Under Exact Maximization}

Having established the practical components of RECLAIM, we now turn to its theoretical justification. The following theorem shows that exact maximization of the score function in \cref{eq:score-large-sample} recovers the true causal structure up to $\si$-Markov equivalence. Specifically, the graph $\hat\gr$ obtained by maximizing \cref{eq:score-large-sample} lies in the same general Markov equivalence class as the ground-truth graph $\gr^\ast$~\citep{bongers2021foundations} across all interventional settings $I_k \in \si$.

\begin{theorem}\label{theorem:consistency}
Let $\si = \{I_k\}_{k=1}^K$ be a family of interventional targets satifying \cref{cond:channel-indentifiability-interventions}, let $\gr^\ast$ denote the ground truth directed graph, $p^{(k)}$ denote the data generating distribution for $I_k \in \si$, and $\hat{\gr} := \arg\max_{\gr} \score(\gr)$. Then, under \cref{app-assume:sufficient-capacity,app-assume:faithfulness,app-assume:positivity,app-assume:finite-entropy,app-assume:parameter-compactness}, and for a suitably chosen $\lambda > 0$, we have that $\hat{\gr} \equiv_\si \gr^\ast$. That is, $\hat{\gr}$ is $\si$-Markov equivalent to $\gr^\ast$.
\end{theorem}

Here, we provide an overview of the key assumptions needed for \cref{theorem:consistency}, as well as proof sketch below. See \cref{app:theorem-consistency-proof} for the full assumptions as well as the complete proof. \Cref{app-assume:sufficient-capacity} ensures that the data-generating distribution lies within the model class, \cref{app-assume:faithfulness} guarantees that all the statistical independencies observed in the data is a result of $\sigma$-separation in the data generating graph~\citep{forre2017markov}. \Cref{app-assume:positivity,app-assume:finite-entropy} prevent the score from diverging to infinity. Finally, \cref{app-assume:parameter-compactness} ensures that the map from latent distribution to the observed distribution is injective. 

% Change this once the full proof is finalized
\begin{proof}[Proof (sketch)]
    Building on the characterization of general directed Markov equivalence class by \citep{bongers2021foundations}, extended to the interventional setting, we show that any graph outside this equivalence class has a strictly lower score than the ground truth graph $\gr^\ast$. This follows from the following two facts: (i) certain independencies present in the data are not captured by graphs outside the equivalence class, and (ii) each latent distribution results in a unique observed distribution. Combined with the expressiveness of the model class, this prevents such graphs from fitting the data properly.
\end{proof}

\section{Experiments}
\label{sec:experiments}

% Colors for the baselines and RECLAIM
\def\colorRECLAIM{blue}
\def\colorNODAGS{green!80!black}
\def\colorDCDI{orange}
\def\colorAnchored{purple}

% The code for RECLAIM is available in the supplementary materials inside the \texttt{code} folder. We plan to make the code public upon pulblication of this work. 

We evaluate RECLAIM on both synthetic and real-world datasets against three state-of-the-art baselines: NODAGS-Flow~\citep{sethuraman2023nodags}, DCDI~\citep{brouillard2020differentiable}, and Anchored-CI~\citep{saeed2020anchored}. Together, these baselines span the key axes of comparison: NODAGS-Flow handles cycles and interventions but not measurement error; DCDI handles interventions but assumes acyclicity and ignores measurement error; and Anchored-CI models measurement error but assumes acyclicity and operates on observational data only. Since Anchored-CI returns a CPDAG representing a Markov equivalence class, we evaluate it on observational data and report the best score across all graphs in the equivalence class.

\begin{figure}[t!]
    \centering
    \begin{minipage}{\textwidth}
        \centering
        \def\markerscale{0.8}
\begin{tikzpicture}
\def\opacitylevel{0.1}

% Table paths
\def\tableanchored{main-data/varying-sigma-gan/pgfplots_anchored_ci.csv}
\def\tabledcdi{main-data/varying-sigma-gan/pgfplots_dcdi.csv}
\def\tablenodags{main-data/varying-sigma-gan/pgfplots_nodags.csv}
\def\tablereclaim{main-data/varying-sigma-gan/pgfplots_reclaim_sir.csv}

\pgfplotsset{
    every axis/.style={
        width=6.2cm,
        height=4cm,
        grid, 
        grid style={
            dashed,
        },
        xlabel={Min noise st. dev ($\sigma_{\min}$)},
        xtick={0.1, 0.3, 0.5, 0.7, 0.9},
        tick label style={font=\footnotesize},
        label style={font=\footnotesize\bfseries},
        legend style={font=\footnotesize},
    }
}
\begin{groupplot}[
    group style={
        group size=2 by 1,
        horizontal sep=1.2cm,
    },
]

% ---- AUPRC ----
\nextgroupplot[
    ylabel={AUPRC},
    ymin=0, ymax=1.15,
    ytick={0,0.25,0.5,0.75,1.0},
    legend to name=varying_sigma_gan_legend,
    legend style={legend columns=4},
]
% Anchored-CI
\addplot[\colorAnchored, name path=aci_au, draw=none, forget plot]
    table[x=n_meas, y=auprc_upper, col sep=comma] {\tableanchored};
\addplot[\colorAnchored, name path=aci_al, draw=none, forget plot]
    table[x=n_meas, y=auprc_lower, col sep=comma] {\tableanchored};
\addplot[\colorAnchored, fill opacity=\opacitylevel, draw=none, forget plot]
    fill between[of=aci_au and aci_al];
\addplot[\colorAnchored, mark=*, mark options={solid, fill=\colorAnchored!100, draw=black, scale=\markerscale}]
    table[x=n_meas, y=auprc_mean, col sep=comma] {\tableanchored};
\addlegendentry{Anchored-CI}

% DCDI
\addplot[\colorDCDI, name path=dcdi_au, draw=none, forget plot]
    table[x=n_meas, y=auprc_upper, col sep=comma] {\tabledcdi};
\addplot[\colorDCDI, name path=dcdi_al, draw=none, forget plot]
    table[x=n_meas, y=auprc_lower, col sep=comma] {\tabledcdi};
\addplot[\colorDCDI, fill opacity=\opacitylevel, draw=none, forget plot]
    fill between[of=dcdi_au and dcdi_al];
\addplot[\colorDCDI, mark=*, mark options={solid, fill=\colorDCDI!100, draw=black, scale=\markerscale}]
    table[x=n_meas, y=auprc_mean, col sep=comma] {\tabledcdi};
\addlegendentry{DCDI}
% NODAGS
\addplot[\colorNODAGS, name path=nodags_au, draw=none, forget plot]
    table[x=n_meas, y=auprc_upper, col sep=comma] {\tablenodags};
\addplot[\colorNODAGS, name path=nodags_al, draw=none, forget plot]
    table[x=n_meas, y=auprc_lower, col sep=comma] {\tablenodags};
\addplot[\colorNODAGS, fill opacity=\opacitylevel, draw=none, forget plot]
    fill between[of=nodags_au and nodags_al];
\addplot[\colorNODAGS, mark=*, mark options={solid, draw=black, fill=\colorNODAGS, scale=\markerscale}]
    table[x=n_meas, y=auprc_mean, col sep=comma] {\tablenodags};
\addlegendentry{NODAGS}
% RECLAIM
\addplot[\colorRECLAIM, name path=reclaim_au, draw=none, forget plot]
    table[x=n_meas, y=auprc_upper, col sep=comma] {\tablereclaim};
\addplot[\colorRECLAIM, name path=reclaim_al, draw=none, forget plot]
    table[x=n_meas, y=auprc_lower, col sep=comma] {\tablereclaim};
\addplot[\colorRECLAIM, fill opacity=\opacitylevel, draw=none, forget plot]
    fill between[of=reclaim_au and reclaim_al];
\addplot[\colorRECLAIM, mark=*, mark options={solid, draw=black, fill=\colorRECLAIM, scale=\markerscale}]
    table[x=n_meas, y=auprc_mean, col sep=comma] {\tablereclaim};
\addlegendentry{RECLAIM (SIR)}

% ---- SHD ----
\nextgroupplot[
    ylabel={SHD},
    ytick={0,8,16,24},
]

% Anchored-CI
\addplot[\colorAnchored, name path=aci_au, draw=none, forget plot]
    table[x=n_meas, y=shd_upper, col sep=comma] {\tableanchored};
\addplot[\colorAnchored, name path=aci_al, draw=none, forget plot]
    table[x=n_meas, y=shd_lower, col sep=comma] {\tableanchored};
\addplot[\colorAnchored, fill opacity=\opacitylevel, draw=none, forget plot]
    fill between[of=aci_au and aci_al];
\addplot[\colorAnchored, mark=*, mark options={solid, fill=\colorAnchored!100, draw=black, scale=\markerscale}]
    table[x=n_meas, y=shd_mean, col sep=comma] {\tableanchored};

% DCDI
\addplot[\colorDCDI, name path=dcdi_su, draw=none, forget plot]
    table[x=n_meas, y=shd_upper, col sep=comma] {\tabledcdi};
\addplot[\colorDCDI, name path=dcdi_sl, draw=none, forget plot]
    table[x=n_meas, y=shd_lower, col sep=comma] {\tabledcdi};
\addplot[\colorDCDI, fill opacity=\opacitylevel, draw=none, forget plot]
    fill between[of=dcdi_su and dcdi_sl];
\addplot[\colorDCDI, mark=*, mark options={solid, draw=black, fill=\colorDCDI, , scale=\markerscale}]
    table[x=n_meas, y=shd_mean, col sep=comma] {\tabledcdi};
% NODAGS
\addplot[\colorNODAGS, name path=nodags_su, draw=none, forget plot]
    table[x=n_meas, y=shd_upper, col sep=comma] {\tablenodags};
\addplot[\colorNODAGS, name path=nodags_sl, draw=none, forget plot]
    table[x=n_meas, y=shd_lower, col sep=comma] {\tablenodags};
\addplot[\colorNODAGS, fill opacity=\opacitylevel, draw=none, forget plot]
    fill between[of=nodags_su and nodags_sl];
\addplot[\colorNODAGS, mark=*, mark options={solid, fill=\colorNODAGS, draw=black, , scale=\markerscale}]
    table[x=n_meas, y=shd_mean, col sep=comma] {\tablenodags};
% RECLAIM
\addplot[\colorRECLAIM, name path=reclaim_su, draw=none, forget plot]
    table[x=n_meas, y=shd_upper, col sep=comma] {\tablereclaim};
\addplot[\colorRECLAIM, name path=reclaim_sl, draw=none, forget plot]
    table[x=n_meas, y=shd_lower, col sep=comma] {\tablereclaim};
\addplot[\colorRECLAIM, fill opacity=\opacitylevel, draw=none, forget plot]
    fill between[of=reclaim_su and reclaim_sl];
\addplot[\colorRECLAIM, mark=*, mark options={solid, fill=\colorRECLAIM, draw=black, , scale=\markerscale}]
    table[x=n_meas, y=shd_mean, col sep=comma] {\tablereclaim};

\end{groupplot}

\node[above=0.2cm, font=\footnotesize\bfseries]
    at ($(group c1r1.north)!0.5!(group c2r1.north)$) {Varying $\sigma_{\min}$ (GAN)};

\end{tikzpicture}
\par\smallskip
% \pgfplotslegendfromname{varying_sigma_gan_legend}

        \vspace{0.7em}
        \def\markerscale{0.8}
\begin{tikzpicture}
\def\opacitylevel{0.1}

% Table paths
\def\tableanchored{main-data/varying-nodes-gan/pgfplots_anchored_ci.csv}
\def\tabledcdi{main-data/varying-nodes-gan/pgfplots_dcdi.csv}
\def\tablenodags{main-data/varying-nodes-gan/pgfplots_nodags.csv}
\def\tablereclaim{main-data/varying-nodes-gan/pgfplots_reclaim_sir.csv}

\pgfplotsset{
    every axis/.style={
        width=6.2cm,
        height=4cm,
        grid, 
        grid style={
            dashed,
        },
        xlabel={Num.\ latent nodes ($d$)},
        xtick={10,20,30,40,50},
        tick label style={font=\footnotesize},
        label style={font=\footnotesize\bfseries},
        legend style={font=\footnotesize},
    }
}
\begin{groupplot}[
    group style={
        group size=2 by 1,
        horizontal sep=1.2cm,
    },
]

% ---- AUPRC ----
\nextgroupplot[
    ylabel={AUPRC},
    ymin=0, ymax=1.15,
    ytick={0,0.25,0.5,0.75,1.0},
    legend to name=varying_nodes_legend,
    legend style={legend columns=4},
]
% Anchored-CI
\addplot[\colorAnchored, name path=aci_au, draw=none, forget plot]
    table[x=n_meas, y=auprc_upper, col sep=comma] {\tableanchored};
\addplot[\colorAnchored, name path=aci_al, draw=none, forget plot]
    table[x=n_meas, y=auprc_lower, col sep=comma] {\tableanchored};
\addplot[\colorAnchored, fill opacity=\opacitylevel, draw=none, forget plot]
    fill between[of=aci_au and aci_al];
\addplot[\colorAnchored, mark=*, mark options={solid, fill=\colorAnchored!100, draw=black, scale=\markerscale}]
    table[x=n_meas, y=auprc_mean, col sep=comma] {\tableanchored};
\addlegendentry{Anchored-CI}

% DCDI
\addplot[\colorDCDI, name path=dcdi_au, draw=none, forget plot]
    table[x=n_meas, y=auprc_upper, col sep=comma] {\tabledcdi};
\addplot[\colorDCDI, name path=dcdi_al, draw=none, forget plot]
    table[x=n_meas, y=auprc_lower, col sep=comma] {\tabledcdi};
\addplot[\colorDCDI, fill opacity=\opacitylevel, draw=none, forget plot]
    fill between[of=dcdi_au and dcdi_al];
\addplot[\colorDCDI, mark=*, mark options={solid, fill=\colorDCDI!100, draw=black, scale=\markerscale}]
    table[x=n_meas, y=auprc_mean, col sep=comma] {\tabledcdi};
\addlegendentry{DCDI}
% NODAGS
\addplot[\colorNODAGS, name path=nodags_au, draw=none, forget plot]
    table[x=n_meas, y=auprc_upper, col sep=comma] {\tablenodags};
\addplot[\colorNODAGS, name path=nodags_al, draw=none, forget plot]
    table[x=n_meas, y=auprc_lower, col sep=comma] {\tablenodags};
\addplot[\colorNODAGS, fill opacity=\opacitylevel, draw=none, forget plot]
    fill between[of=nodags_au and nodags_al];
\addplot[\colorNODAGS, mark=*, mark options={solid, draw=black, fill=\colorNODAGS, scale=\markerscale}]
    table[x=n_meas, y=auprc_mean, col sep=comma] {\tablenodags};
\addlegendentry{NODAGS}
% RECLAIM
\addplot[\colorRECLAIM, name path=reclaim_au, draw=none, forget plot]
    table[x=n_meas, y=auprc_upper, col sep=comma] {\tablereclaim};
\addplot[\colorRECLAIM, name path=reclaim_al, draw=none, forget plot]
    table[x=n_meas, y=auprc_lower, col sep=comma] {\tablereclaim};
\addplot[\colorRECLAIM, fill opacity=\opacitylevel, draw=none, forget plot]
    fill between[of=reclaim_au and reclaim_al];
\addplot[\colorRECLAIM, mark=*, mark options={solid, draw=black, fill=\colorRECLAIM, scale=\markerscale}]
    table[x=n_meas, y=auprc_mean, col sep=comma] {\tablereclaim};
\addlegendentry{RECLAIM (SIR)}

% ---- SHD ----
\nextgroupplot[
    ylabel={SHD},
    ytick={0,50,100,150},
]

% Anchored-CI
\addplot[\colorAnchored, name path=aci_au, draw=none, forget plot]
    table[x=n_meas, y=shd_upper, col sep=comma] {\tableanchored};
\addplot[\colorAnchored, name path=aci_al, draw=none, forget plot]
    table[x=n_meas, y=shd_lower, col sep=comma] {\tableanchored};
\addplot[\colorAnchored, fill opacity=\opacitylevel, draw=none, forget plot]
    fill between[of=aci_au and aci_al];
\addplot[\colorAnchored, mark=*, mark options={solid, fill=\colorAnchored!100, draw=black, scale=\markerscale}]
    table[x=n_meas, y=shd_mean, col sep=comma] {\tableanchored};

% DCDI
\addplot[\colorDCDI, name path=dcdi_su, draw=none, forget plot]
    table[x=n_meas, y=shd_upper, col sep=comma] {\tabledcdi};
\addplot[\colorDCDI, name path=dcdi_sl, draw=none, forget plot]
    table[x=n_meas, y=shd_lower, col sep=comma] {\tabledcdi};
\addplot[\colorDCDI, fill opacity=\opacitylevel, draw=none, forget plot]
    fill between[of=dcdi_su and dcdi_sl];
\addplot[\colorDCDI, mark=*, mark options={solid, draw=black, fill=\colorDCDI, , scale=\markerscale}]
    table[x=n_meas, y=shd_mean, col sep=comma] {\tabledcdi};
% NODAGS
\addplot[\colorNODAGS, name path=nodags_su, draw=none, forget plot]
    table[x=n_meas, y=shd_upper, col sep=comma] {\tablenodags};
\addplot[\colorNODAGS, name path=nodags_sl, draw=none, forget plot]
    table[x=n_meas, y=shd_lower, col sep=comma] {\tablenodags};
\addplot[\colorNODAGS, fill opacity=\opacitylevel, draw=none, forget plot]
    fill between[of=nodags_su and nodags_sl];
\addplot[\colorNODAGS, mark=*, mark options={solid, fill=\colorNODAGS, draw=black, , scale=\markerscale}]
    table[x=n_meas, y=shd_mean, col sep=comma] {\tablenodags};
% RECLAIM
\addplot[\colorRECLAIM, name path=reclaim_su, draw=none, forget plot]
    table[x=n_meas, y=shd_upper, col sep=comma] {\tablereclaim};
\addplot[\colorRECLAIM, name path=reclaim_sl, draw=none, forget plot]
    table[x=n_meas, y=shd_lower, col sep=comma] {\tablereclaim};
\addplot[\colorRECLAIM, fill opacity=\opacitylevel, draw=none, forget plot]
    fill between[of=reclaim_su and reclaim_sl];
\addplot[\colorRECLAIM, mark=*, mark options={solid, fill=\colorRECLAIM, draw=black, , scale=\markerscale}]
    table[x=n_meas, y=shd_mean, col sep=comma] {\tablereclaim};

\end{groupplot}

\node[above=0.2cm, font=\footnotesize\bfseries]
    at ($(group c1r1.north)!0.5!(group c2r1.north)$) {Varying Number of Nodes $d$ (GAN)};
\end{tikzpicture}
\par\smallskip\smallskip
\hspace{1.25cm}\pgfplotslegendfromname{varying_nodes_legend}
    \end{minipage}
    \caption{Performance comparison with varying minimum noise standard deviation ($\sigma$) (top), and varying number of latent nodes ($d$) (bottom) for \emph{Gaussian additive noise} system. Shaded regions show $\pm 1$ standard deviation over 10 trials.}
    \label{fig:results-gan}
\end{figure}
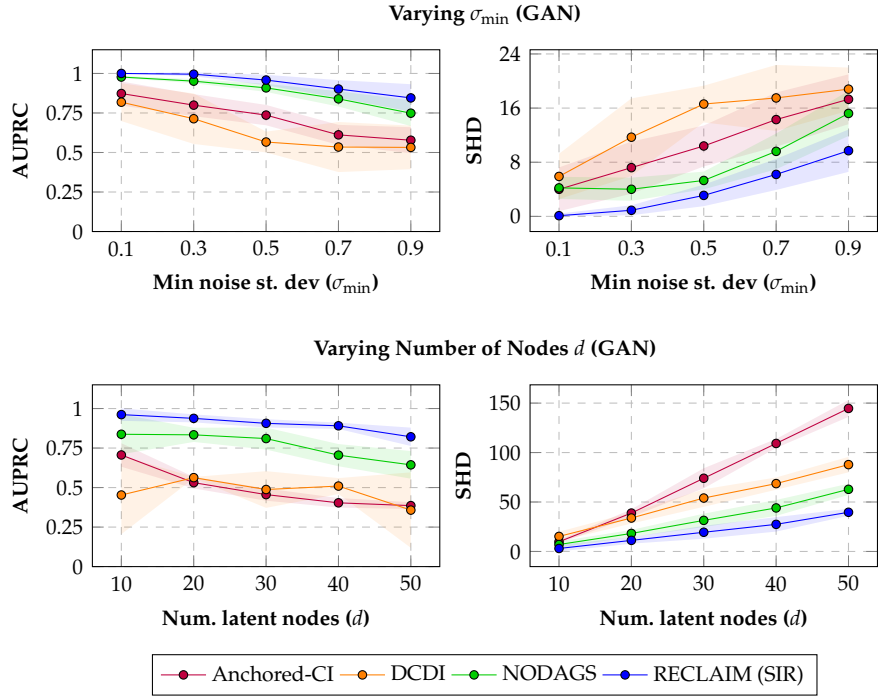

\subsection{Synthetic Experiments}

In all our synthetic experiments, cyclic graphs were generated using the Erd\H{o}s-R\'enyi (ER) random graph model with the expected outgoing edge density set to 2. RECLAIM and the baselines were evaluated on nonlinear SCMs with the latent variables generated using the following structural equations:
$$\vx = \tanh(\mW^\top \vx) + \vz,$$
where the weighted adjacency matrix is faithful to the graph generated using the ER model, and $0.2 \leq |W_{ij}| \leq 0.9$. The matrix $\mW$ was rescaled to ensure that the causal mechanism remains contractive. The training dataset consists of observational data and single node interventions over all the nodes in the graph, i.e., $\si = \emptyset \cup \{\{i\}\}_{i=1}^d$. For each interventional experiment $I_k \in \si$, the intervened nodes $X_i \in \sx_{I_k}$ were sampled from $\gauss(0, 1)$, with $N_k=1000$ samples per intervention. 

We report performance using two metrics: AUPRC (higher is better) and SHD (lower is better). AUPRC summarizes edge recovery quality across thresholds, while SHD counts the number of edge additions, deletions, and reversals needed to recover the ground truth graph. Since SHD requires a binary adjacency matrix, we threshold the estimated edge probabilities at 0.8.

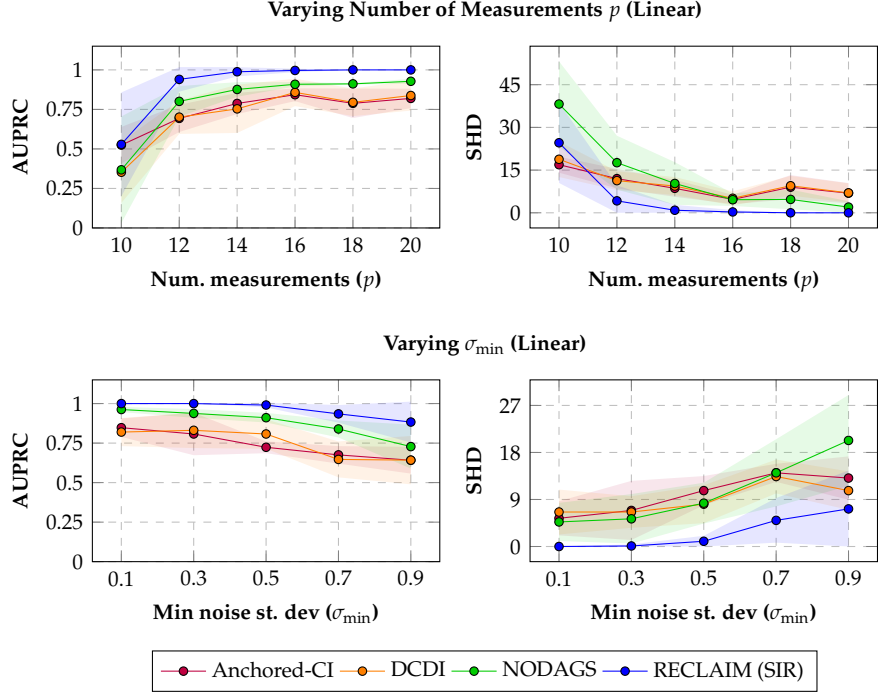
\begin{figure}[t!]
    \centering
    \def\markerscale{0.8}
\begin{tikzpicture}
\def\opacitylevel{0.1}

% Table paths
\def\tableanchored{main-data/varying-measurements-linear/pgfplots_anchored_ci.csv}
\def\tabledcdi{main-data/varying-measurements-linear/pgfplots_dcdi.csv}
\def\tablenodags{main-data/varying-measurements-linear/pgfplots_nodags.csv}
\def\tablereclaim{main-data/varying-measurements-linear/pgfplots_reclaim_sir.csv}

\pgfplotsset{
    every axis/.style={
        width=6.2cm,
        height=4cm,
        grid, 
        grid style={
            dashed,
        },
        xlabel={Num.\ measurements ($p$)},
        xtick={10,12,14,16,18,20},
        tick label style={font=\footnotesize},
        label style={font=\footnotesize\bfseries},
        legend style={font=\footnotesize},
    }
}
\begin{groupplot}[
    group style={
        group size=2 by 1,
        horizontal sep=1.2cm,
    },
]

% ---- AUPRC ----
\nextgroupplot[
    ylabel={AUPRC},
    ymin=0, ymax=1.15,
    ytick={0,0.25,0.5,0.75,1.0},
    legend to name=varying_meas_legend,
    legend style={legend columns=4},
]
% Anchored-CI
\addplot[\colorAnchored, name path=aci_au, draw=none, forget plot]
    table[x=n_meas, y=auprc_upper, col sep=comma] {\tableanchored};
\addplot[\colorAnchored, name path=aci_al, draw=none, forget plot]
    table[x=n_meas, y=auprc_lower, col sep=comma] {\tableanchored};
\addplot[\colorAnchored, fill opacity=\opacitylevel, draw=none, forget plot]
    fill between[of=aci_au and aci_al];
\addplot[\colorAnchored, mark=*, mark options={solid, fill=\colorAnchored!100, draw=black, scale=\markerscale}]
    table[x=n_meas, y=auprc_mean, col sep=comma] {\tableanchored};
\addlegendentry{Anchored-CI}

% DCDI
\addplot[\colorDCDI, name path=dcdi_au, draw=none, forget plot]
    table[x=n_meas, y=auprc_upper, col sep=comma] {\tabledcdi};
\addplot[\colorDCDI, name path=dcdi_al, draw=none, forget plot]
    table[x=n_meas, y=auprc_lower, col sep=comma] {\tabledcdi};
\addplot[\colorDCDI, fill opacity=\opacitylevel, draw=none, forget plot]
    fill between[of=dcdi_au and dcdi_al];
\addplot[\colorDCDI, mark=*, mark options={solid, fill=\colorDCDI!100, draw=black, scale=\markerscale}]
    table[x=n_meas, y=auprc_mean, col sep=comma] {\tabledcdi};
\addlegendentry{DCDI}
% NODAGS
\addplot[\colorNODAGS, name path=nodags_au, draw=none, forget plot]
    table[x=n_meas, y=auprc_upper, col sep=comma] {\tablenodags};
\addplot[\colorNODAGS, name path=nodags_al, draw=none, forget plot]
    table[x=n_meas, y=auprc_lower, col sep=comma] {\tablenodags};
\addplot[\colorNODAGS, fill opacity=\opacitylevel, draw=none, forget plot]
    fill between[of=nodags_au and nodags_al];
\addplot[\colorNODAGS, mark=*, mark options={solid, draw=black, fill=\colorNODAGS, scale=\markerscale}]
    table[x=n_meas, y=auprc_mean, col sep=comma] {\tablenodags};
\addlegendentry{NODAGS}
% RECLAIM
\addplot[\colorRECLAIM, name path=reclaim_au, draw=none, forget plot]
    table[x=n_meas, y=auprc_upper, col sep=comma] {\tablereclaim};
\addplot[\colorRECLAIM, name path=reclaim_al, draw=none, forget plot]
    table[x=n_meas, y=auprc_lower, col sep=comma] {\tablereclaim};
\addplot[\colorRECLAIM, fill opacity=\opacitylevel, draw=none, forget plot]
    fill between[of=reclaim_au and reclaim_al];
\addplot[\colorRECLAIM, mark=*, mark options={solid, draw=black, fill=\colorRECLAIM, scale=\markerscale}]
    table[x=n_meas, y=auprc_mean, col sep=comma] {\tablereclaim};
\addlegendentry{RECLAIM (SIR)}

% ---- SHD ----
\nextgroupplot[
    ylabel={SHD},
    ytick={0,15,30,45},
]

% Anchored-CI
\addplot[\colorAnchored, name path=aci_au, draw=none, forget plot]
    table[x=n_meas, y=shd_upper, col sep=comma] {\tableanchored};
\addplot[\colorAnchored, name path=aci_al, draw=none, forget plot]
    table[x=n_meas, y=shd_lower, col sep=comma] {\tableanchored};
\addplot[\colorAnchored, fill opacity=\opacitylevel, draw=none, forget plot]
    fill between[of=aci_au and aci_al];
\addplot[\colorAnchored, mark=*, mark options={solid, fill=\colorAnchored!100, draw=black, scale=\markerscale}]
    table[x=n_meas, y=shd_mean, col sep=comma] {\tableanchored};

% DCDI
\addplot[\colorDCDI, name path=dcdi_su, draw=none, forget plot]
    table[x=n_meas, y=shd_upper, col sep=comma] {\tabledcdi};
\addplot[\colorDCDI, name path=dcdi_sl, draw=none, forget plot]
    table[x=n_meas, y=shd_lower, col sep=comma] {\tabledcdi};
\addplot[\colorDCDI, fill opacity=\opacitylevel, draw=none, forget plot]
    fill between[of=dcdi_su and dcdi_sl];
\addplot[\colorDCDI, mark=*, mark options={solid, draw=black, fill=\colorDCDI, , scale=\markerscale}]
    table[x=n_meas, y=shd_mean, col sep=comma] {\tabledcdi};
% NODAGS
\addplot[\colorNODAGS, name path=nodags_su, draw=none, forget plot]
    table[x=n_meas, y=shd_upper, col sep=comma] {\tablenodags};
\addplot[\colorNODAGS, name path=nodags_sl, draw=none, forget plot]
    table[x=n_meas, y=shd_lower, col sep=comma] {\tablenodags};
\addplot[\colorNODAGS, fill opacity=\opacitylevel, draw=none, forget plot]
    fill between[of=nodags_su and nodags_sl];
\addplot[\colorNODAGS, mark=*, mark options={solid, fill=\colorNODAGS, draw=black, , scale=\markerscale}]
    table[x=n_meas, y=shd_mean, col sep=comma] {\tablenodags};
% RECLAIM
\addplot[\colorRECLAIM, name path=reclaim_su, draw=none, forget plot]
    table[x=n_meas, y=shd_upper, col sep=comma] {\tablereclaim};
\addplot[\colorRECLAIM, name path=reclaim_sl, draw=none, forget plot]
    table[x=n_meas, y=shd_lower, col sep=comma] {\tablereclaim};
\addplot[\colorRECLAIM, fill opacity=\opacitylevel, draw=none, forget plot]
    fill between[of=reclaim_su and reclaim_sl];
\addplot[\colorRECLAIM, mark=*, mark options={solid, fill=\colorRECLAIM, draw=black, , scale=\markerscale}]
    table[x=n_meas, y=shd_mean, col sep=comma] {\tablereclaim};

\end{groupplot}

\node[above=0.2cm, font=\footnotesize\bfseries]
    at ($(group c1r1.north)!0.5!(group c2r1.north)$) {Varying Number of Measurements $p$ (Linear)};

\end{tikzpicture}
\par\smallskip
% \pgfplotslegendfromname{varying_meas_legend}

    \vspace{0.2cm}
    \def\markerscale{0.8}
\begin{tikzpicture}
\def\opacitylevel{0.1}

% Table paths
\def\tableanchored{main-data/varying-sigma-linear/pgfplots_anchored_ci.csv}
\def\tabledcdi{main-data/varying-sigma-linear/pgfplots_dcdi.csv}
\def\tablenodags{main-data/varying-sigma-linear/pgfplots_nodags.csv}
\def\tablereclaim{main-data/varying-sigma-linear/pgfplots_reclaim_sir.csv}

\pgfplotsset{
    every axis/.style={
        width=6.2cm,
        height=4cm,
        grid, 
        grid style={
            dashed,
        },
        xlabel={Min noise st. dev ($\sigma_{\min}$)},
        xtick={0.1, 0.3, 0.5, 0.7, 0.9},
        tick label style={font=\footnotesize},
        label style={font=\footnotesize\bfseries},
        legend style={font=\footnotesize},
    }
}
\begin{groupplot}[
    group style={
        group size=2 by 1,
        horizontal sep=1.2cm,
    },
]

% ---- AUPRC ----
\nextgroupplot[
    ylabel={AUPRC},
    ymin=0, ymax=1.15,
    ytick={0,0.25,0.5,0.75,1.0},
    legend to name=varying_sigma_linear_legend,
    legend style={legend columns=4},
]
% Anchored-CI
\addplot[\colorAnchored, name path=aci_au, draw=none, forget plot]
    table[x=n_meas, y=auprc_upper, col sep=comma] {\tableanchored};
\addplot[\colorAnchored, name path=aci_al, draw=none, forget plot]
    table[x=n_meas, y=auprc_lower, col sep=comma] {\tableanchored};
\addplot[\colorAnchored, fill opacity=\opacitylevel, draw=none, forget plot]
    fill between[of=aci_au and aci_al];
\addplot[\colorAnchored, mark=*, mark options={solid, fill=\colorAnchored!100, draw=black, scale=\markerscale}]
    table[x=n_meas, y=auprc_mean, col sep=comma] {\tableanchored};
\addlegendentry{Anchored-CI}

% DCDI
\addplot[\colorDCDI, name path=dcdi_au, draw=none, forget plot]
    table[x=n_meas, y=auprc_upper, col sep=comma] {\tabledcdi};
\addplot[\colorDCDI, name path=dcdi_al, draw=none, forget plot]
    table[x=n_meas, y=auprc_lower, col sep=comma] {\tabledcdi};
\addplot[\colorDCDI, fill opacity=\opacitylevel, draw=none, forget plot]
    fill between[of=dcdi_au and dcdi_al];
\addplot[\colorDCDI, mark=*, mark options={solid, fill=\colorDCDI!100, draw=black, scale=\markerscale}]
    table[x=n_meas, y=auprc_mean, col sep=comma] {\tabledcdi};
\addlegendentry{DCDI}
% NODAGS
\addplot[\colorNODAGS, name path=nodags_au, draw=none, forget plot]
    table[x=n_meas, y=auprc_upper, col sep=comma] {\tablenodags};
\addplot[\colorNODAGS, name path=nodags_al, draw=none, forget plot]
    table[x=n_meas, y=auprc_lower, col sep=comma] {\tablenodags};
\addplot[\colorNODAGS, fill opacity=\opacitylevel, draw=none, forget plot]
    fill between[of=nodags_au and nodags_al];
\addplot[\colorNODAGS, mark=*, mark options={solid, draw=black, fill=\colorNODAGS, scale=\markerscale}]
    table[x=n_meas, y=auprc_mean, col sep=comma] {\tablenodags};
\addlegendentry{NODAGS}
% RECLAIM
\addplot[\colorRECLAIM, name path=reclaim_au, draw=none, forget plot]
    table[x=n_meas, y=auprc_upper, col sep=comma] {\tablereclaim};
\addplot[\colorRECLAIM, name path=reclaim_al, draw=none, forget plot]
    table[x=n_meas, y=auprc_lower, col sep=comma] {\tablereclaim};
\addplot[\colorRECLAIM, fill opacity=\opacitylevel, draw=none, forget plot]
    fill between[of=reclaim_au and reclaim_al];
\addplot[\colorRECLAIM, mark=*, mark options={solid, draw=black, fill=\colorRECLAIM, scale=\markerscale}]
    table[x=n_meas, y=auprc_mean, col sep=comma] {\tablereclaim};
\addlegendentry{RECLAIM (SIR)}

% ---- SHD ----
\nextgroupplot[
    ylabel={SHD},
    ytick={0,9,18,27},
]

% Anchored-CI
\addplot[\colorAnchored, name path=aci_au, draw=none, forget plot]
    table[x=n_meas, y=shd_upper, col sep=comma] {\tableanchored};
\addplot[\colorAnchored, name path=aci_al, draw=none, forget plot]
    table[x=n_meas, y=shd_lower, col sep=comma] {\tableanchored};
\addplot[\colorAnchored, fill opacity=\opacitylevel, draw=none, forget plot]
    fill between[of=aci_au and aci_al];
\addplot[\colorAnchored, mark=*, mark options={solid, fill=\colorAnchored!100, draw=black, scale=\markerscale}]
    table[x=n_meas, y=shd_mean, col sep=comma] {\tableanchored};

% DCDI
\addplot[\colorDCDI, name path=dcdi_su, draw=none, forget plot]
    table[x=n_meas, y=shd_upper, col sep=comma] {\tabledcdi};
\addplot[\colorDCDI, name path=dcdi_sl, draw=none, forget plot]
    table[x=n_meas, y=shd_lower, col sep=comma] {\tabledcdi};
\addplot[\colorDCDI, fill opacity=\opacitylevel, draw=none, forget plot]
    fill between[of=dcdi_su and dcdi_sl];
\addplot[\colorDCDI, mark=*, mark options={solid, draw=black, fill=\colorDCDI, , scale=\markerscale}]
    table[x=n_meas, y=shd_mean, col sep=comma] {\tabledcdi};
% NODAGS
\addplot[\colorNODAGS, name path=nodags_su, draw=none, forget plot]
    table[x=n_meas, y=shd_upper, col sep=comma] {\tablenodags};
\addplot[\colorNODAGS, name path=nodags_sl, draw=none, forget plot]
    table[x=n_meas, y=shd_lower, col sep=comma] {\tablenodags};
\addplot[\colorNODAGS, fill opacity=\opacitylevel, draw=none, forget plot]
    fill between[of=nodags_su and nodags_sl];
\addplot[\colorNODAGS, mark=*, mark options={solid, fill=\colorNODAGS, draw=black, , scale=\markerscale}]
    table[x=n_meas, y=shd_mean, col sep=comma] {\tablenodags};
% RECLAIM
\addplot[\colorRECLAIM, name path=reclaim_su, draw=none, forget plot]
    table[x=n_meas, y=shd_upper, col sep=comma] {\tablereclaim};
\addplot[\colorRECLAIM, name path=reclaim_sl, draw=none, forget plot]
    table[x=n_meas, y=shd_lower, col sep=comma] {\tablereclaim};
\addplot[\colorRECLAIM, fill opacity=\opacitylevel, draw=none, forget plot]
    fill between[of=reclaim_su and reclaim_sl];
\addplot[\colorRECLAIM, mark=*, mark options={solid, fill=\colorRECLAIM, draw=black, , scale=\markerscale}]
    table[x=n_meas, y=shd_mean, col sep=comma] {\tablereclaim};

\end{groupplot}

\node[above=0.2cm, font=\footnotesize\bfseries]
    at ($(group c1r1.north)!0.5!(group c2r1.north)$) {Varying $\sigma_{\min}$ (Linear)};

\end{tikzpicture}
\par\smallskip\smallskip
\hspace{1.25cm}\pgfplotslegendfromname{varying_sigma_linear_legend}

    \caption{Performance comparison with varying number of measurements ($p$) (top), and varying minimum noise standard deviation ($\sigma$) (bottom) for \emph{Linear measurement} system. Shaded regions show $\pm 1$ standard deviation over 10 trials.}
    \label{fig:results-linear}
\end{figure}

\subsubsection{Gaussian Additive Noise.}  

We now consider the Gaussian additive noise setting, where measurements are corrupted as described in \cref{eq:gan-measurements}. We evaluate all methods as a function of the minimum measurement noise scale $\sigma_{\min}$ and the number of latent nodes $d$, probing robustness and scalability respectively, with results shown in \cref{fig:results-gan}.

\paragraph{Sensitivity to measurement noise scale.} We fix $d=10$ and vary $\sigma_{\min}$ between 0.1 and 0.9, with $\sigma_{\max} = \sigma_{\min} + 0.3$. As shown in \cref{fig:results-gan} (top), all methods degrade as $\sigma_{\min}$ increases, but RECLAIM degrades more gracefully, the performance gap over the baselines widens with noise scale, highlighting the benefit of explicitly modeling measurement error. NODAGS-Flow matches the performance of RECLAIM when $\sigma_{\min} = 0.1$ as the influence of noise is very low at that scale. 

\paragraph{Impact of number of latent nodes.} We fix $\sigma_{\min} = 0.5$ and $\sigma_{\max} = 1$, and vary $d$ between 10 and 50. As shown in \cref{fig:results-gan} (bottom), performance degrades for all methods as $d$ increases, but the gap between cyclic methods (RECLAIM and NODAGS-Flow) and acyclic methods (DCDI and Anchored-CI) widens substantially, suggesting that the cost of misspecifying graph structure compounds with graph size.

\subsubsection{Linear Measurement System. }  

We now consider the linear measurement system setting, where measurements follow \cref{eq:linear-measurements} with $A_{ij} \sim \mathcal{N}(0, 1.5)$. We evaluate all methods as a function of the number of measurements $p$ and the minimum noise scale $\sigma_{\min}$, probing the effect of measurement redundancy and noise robustness respectively, with results shown in \cref{fig:results-linear}.

\paragraph{Impact of number of measurements.} We fix $d=10$, $\sigma_{\min}=0.3$, $\sigma_{\max} = 0.6$, and vary $p$ between 10 and 20. As shown in \cref{fig:results-linear} (top), all methods perform poorly at $p = d = 10$, but RECLAIM recovers sharply as $p$ increases, achieving near-perfect recovery for $p \geq 12$. This suggests that even modest measurement redundancy is sufficient for RECLAIM to disentangle the latent structure, whereas the baselines continue to struggle across the full range of $p$.

\paragraph{Sensitivity to measurement noise scale.} We fix $p = 15, d = 10$ and vary $\sigma_{\min}$ between 0.1 and 0.9, with $\sigma_{\max} = \sigma_{\min} + 0.3$. As shown in \cref{fig:results-linear} (bottom), the overall trend mirrors the Gaussian additive noise setting (all methods degrade with increasing noise), but RECLAIM degrades more gracefully than the baselines. Notably, RECLAIM is more robust here than in the Gaussian additive noise setting (AUPRC of 0.88 vs 0.84 at $\sigma_{\min} = 0.9$), which we attribute to the additional structure provided by the linear measurement system allowing for more accurate noise parameter estimation.

\begin{figure}[t]
    \centering
        \scalebox{1.0}{
            \scalebox{0.8}{
\begin{tikzpicture}[
    every node/.style={
        ellipse, 
        draw=black,
        fill=black!10,
        font=\scriptsize
    },
    myarrow,
    bend angle=15
]

% 11 nodes arranged in a circle (radius = 3cm)
% Node i is at angle: 90 - i * (360/11) degrees

% Nodes with biological labels
\node (0)  at ({90 - 0  * (360/11)}:2.5cm) {Raf};
\node (1)  at ({90 - 1  * (360/11)}:2.5cm) {Mek};
\node (2)  at ({90 - 2  * (360/11)}:2.5cm) {PLCg};
\node (3)  at ({90 - 3  * (360/11)}:2.5cm) {PIP2};
\node (4)  at ({90 - 4  * (360/11)}:2.5cm) {PIP3};
\node (5)  at ({90 - 5  * (360/11)}:2.5cm) {ERK};
\node (6)  at ({90 - 6  * (360/11)}:2.5cm) {Akt};
\node (7)  at ({90 - 7  * (360/11)}:2.5cm) {PKA};
\node (8)  at ({90 - 8  * (360/11)}:2.5cm) {PKC};
\node (9)  at ({90 - 9  * (360/11)}:2.5cm) {P38};
\node (10) at ({90 - 10 * (360/11)}:2.5cm) {PJNK};

% Edges (bidirectional pairs use bend to separate arrows)
% 0 <-> 1
\draw[bend left, blue]  (0) to (1);
\draw[bend left, blue]  (1) to (0);

% 2 -> 4, 3 -> 2, 3 -> 4, 4 -> 3
\draw[bend right, blue] (2) to (4);
\draw[blue] (3) to (2);
\draw[bend left, blue]  (3) to (4);
\draw[bend left, blue]  (4) to (3);

% 5 <-> 6
\draw[bend left, blue]  (5) to (6);
\draw[bend left, blue]  (6) to (5);

% 6 -> 7, 7 -> 5, 7 -> 6
\draw[bend left, blue]  (6) to (7);
\draw[bend left, blue]  (7) to (6);
\draw[bend left, blue] (7) to (5);

% 8 -> 9, 8 -> 10, 9 -> 8, 10 -> 9
\draw[bend left, blue]  (8) to (9);
\draw[bend left, blue]  (9) to (8);
\draw[bend right, blue] (8) to (10);
\draw[blue] (10) to (9);

\end{tikzpicture}
}
        }
        \caption{Estimated graph learnt from~\cite{sachs_causal_2005} dataset}
        \label{fig:sachs-graph}
\end{figure}
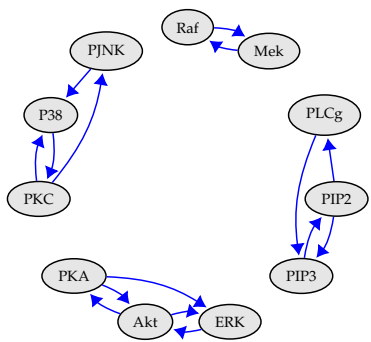

\subsection{Real-World Protein Signaling Dataset}

We evaluate RECLAIM and the baselines on the~\cite{sachs_causal_2005} protein signaling dataset, which consists of fluorescence intensity measurements of phosphorylated proteins and phospholipid components in human immune system cells across 13 interventional environments. These measurements are noisy proxies for true protein levels, corrupted by instrument and antibody staining noise~\citep{krutzik2004analysis}. Importantly, 11 proteins and phospholipids are measured, making this a natural testbed for causal discovery under measurement noise.

\begin{wraptable}{r}{0.35\textwidth}
    \centering
    % \vspace{cm}
    \caption{Performance with respect to SHD on~\citet{sachs_causal_2005} dataset}
    \vspace{0.1cm}
    \label{tab:sachs-results}
    \begin{tabular}{cc}
        \toprule
        \textbf{Method} & \textbf{SHD}\\
        \midrule 
        RECLAIM & 19\\
        NODAGS & 22\\
        DCDI & 21\\
        Anchored-CI & 19\\
        \bottomrule
    \end{tabular}
    \vspace{-0.2cm}
\end{wraptable}

We train RECLAIM on the first 9 interventional environments. Since the true variance of the intervened proteins is unavailable, we treat the measurement noise variance as a learnable parameter. The estimated signaling network after 200 epochs is shown in \cref{fig:sachs-graph}, with SHD scores relative to the~\citet{sachs_causal_2005} ground-truth reported in \cref{tab:sachs-results}.

The Raf$\leftrightarrow$Mek cycle recovered by RECLAIM is consistent with the well-documented negative feedback from ERK to Raf-1~\citep{sturm2011massive}, which DAG constrained methods structurally cannot recover. 
We note that the ground-truth graph from~\citet{sachs_causal_2005} is itself a DAG, despite feedback loops being known to exist in these systems~\citep{sturm2011massive}---meaning our SHD scores are conservative with respect to RECLAIM's true recovery performance.

\paragraph{Additional experiments.} Additionally, we also provide results in \cref{app:additional-experiments} for the following settings: (i) varying number of cycles in latent graph, (ii) varying degree on nonlinearity of latent SCM, and (iii) varying sparsity of the latent graph.

\section{Discussion and Conclusion}
\label{sec:discussions}

In this work, we introduced RECLAIM, a differentiable causal discovery framework that can jointly handle cyclic causal graphs, interventional data, and measurement noise in a unified probabilistic framework. Unlike existing methods that assume acyclicity or direct observation of system variables, RECLAIM operates on coarsened measurements across two noise settings: additive Gaussian, and linear measurement systems. We establish consistency of the graph estimator in the large-sample regime and demonstrate through experiments that RECLAIM outperforms state-of-the-art baselines, with particular gains in cyclic and high-noise regimes.

A key limitation is the reliance on interventional data for measurement noise parameter 
estimation---insufficient interventional coverage can degrade both noise estimation and 
graph recovery. Future directions include relaxing this requirement, extending support to 
soft and unknown interventional targets, non-Gaussian exogenous noise, and jointly 
learning the measurement system from data.

\subsection*{Acknowledgment}

This material is based on work supported by the National Science Foundation (NSF) under Grant no. 2502298. 

\bibliography{references}

\include{appendix}

\end{document}